\documentclass{article}

\usepackage{microtype}
\usepackage{bbm}
\usepackage{arydshln}
\usepackage{dsfont}
\usepackage{graphicx}
\usepackage{subcaption}
\usepackage{stfloats}
\usepackage{booktabs} 

\usepackage{hyperref}

\usepackage{breakurl}

\newcommand{\TV}{\text{TV}}
\newcommand{\R}{\mathbb{R}}

\newcommand{\newinf}{\mathop{\mathrm{inf}\vphantom{\mathrm{sup}}}}

\usepackage[accepted]{icml2023}

\usepackage{amsmath}
\usepackage{amssymb}
\usepackage{mathtools}
\usepackage{amsthm}

\usepackage[capitalize,noabbrev]{cleveref}

\theoremstyle{plain}
\newtheorem{theorem}{Theorem}[section]

\newtheorem{algo}[theorem]{Algorithm}
\theoremstyle{definition}
\newtheorem{definition}[theorem]{Definition}

\theoremstyle{remark}

\usepackage[textsize=tiny]{todonotes}

\icmltitlerunning{Baselines for Identifying Watermarked Large Language Models}

\begin{document}

\twocolumn[
\icmltitle{Baselines for Identifying Watermarked Large Language Models}




\begin{icmlauthorlist}
\icmlauthor{Leonard Tang}{yyy}
\icmlauthor{Gavin Uberti}{yyy}
\icmlauthor{Tom Shlomi}{yyy}
\end{icmlauthorlist}

\icmlaffiliation{yyy}{Department of Computer Science, Harvard University}

\icmlcorrespondingauthor{Leonard Tang}{leonardtang@college.harvard.edu}

\icmlkeywords{Machine Learning, ICML}

\vskip 0.3in
]



\printAffiliationsAndNotice{}  

\begin{abstract}
We consider the emerging problem of identifying the presence and use of watermarking schemes in widely used, publicly hosted, closed source large language models (LLMs). We introduce a suite of baseline algorithms for identifying watermarks in LLMs that rely on analyzing distributions of output tokens and logits generated by watermarked and unmarked LLMs. Notably, watermarked LLMs tend to produce distributions that diverge qualitatively and identifiably from standard models. Furthermore, we investigate the identifiability of watermarks at varying strengths and consider the tradeoffs of each of our identification mechanisms with respect to watermarking scenario. Along the way, we formalize the specific problem of identifying watermarks in LLMs, as well as LLM watermarks and watermark detection in general, providing a framework and foundations for studying them. 
\end{abstract}

\section{Introduction}

Recent progress in large language models (LLMs) has resulted in a rapid increase in the ability of models to produce convincingly human-like text, sparking worries that LLMs could be used to spread disinformation, enable plagiarism, and maliciously impersonate people. As such, researchers have begun to develop methods to detect AI generated text. These include watermarking algorithms, which subtly modify the outputted text to allow for better detection, given that the detector has sufficient access to watermarking parameters. Differing from previous work, where the focus is on determining if text has been produced by a watermarked model, here we study the problem of if a \textit{language model} has been watermarked. Critically, our black-box algorithms only require querying the model and do not necessitate any knowledge of underlying watermarking parameters.  



\section{Related Work}
\label{sec:related-work}

\paragraph{Generated Text Detection Via Statistical Discrepancies}
Recent methods such as DetectGPT and GPTZero distinguish between machine-generated and human-written text by analyzing their statistical discrepancies \citep{gptzero, mitchell2023detectgpt}. DetectGPT compares the log probability computed by a model on unperturbed text and perturbed variations, leveraging the observation that text sampled from a LLM generally occupy negative curvature regions of the model's log probability function. GPTZero instead uses perplexity and burstiness to distinguish human from machine text, with lower perplexity and burstiness indicating a greater likelihood of machine-generated text. However, these heuristics do not generalize and are often fallible.

\paragraph{Detection by Learning Classifiers} Several papers have proposed to train classifiers to distinguish between AI and human generated text. During the initial GPT-2 release, OpenAI trained a RoBERTa classifier to detect GPT-2 generated text with 95\% accuracy \cite{solaiman2019release}. More recently, OpenAI fine-tuned a GPT model on a dataset of machine-generated and human texts focusing on the same topic, with a true positive identification rate of 26\% \cite{OpenAIClassifier}. Similarly, \citet{guo2023close} collected the Human ChatGPT Comparison Corpuse (HC3) and fine-tuned RoBERTa for the detection task.


Notably, the capabilities of such classifiers decrease as machine-generated text becomes increasingly human-like. \citet{sadasivan2023can} show theoretically that for sufficiently advanced language models, machine-generated text detectors offer only a marginal improvement over random classifiers. Moreover, such methods are prone to adversarial attacks and are not robust to out-of-distribution text.


\paragraph{Watermarking Large Language Models}
An alternative to detecting of machine-generated text using statistical discrepancies and learned classifiers is the concept of watermarks. Watermarks are hidden patterns in machine-generated text that are imperceptible to humans, but algorithmically identifiable as synthetic. Natural language watermarks long predate the development of LLMs, relying on methods such as synonym substitution, as well as syntactic and semantic transformations \cite{topkara2005natural}. While these schemes are capable of preserving syntactic and semantic meaning, they often break stylistic constraints.



More recently, \citet{kirchenbauer2023watermark} propose a watermarking scheme that minimizes degradation in the quality of generated text, while being efficient to detect in text. In unpublished work, \citet{aaronson2023} introduces a conceptually similar watermarking scheme. At any given inference step, both watermarking approaches modify the output token probabilities of the underlying model with an algorithm using a secret key, hashing, and pseudorandom function properties. We broadly refer to both of these watermarks as \textit{Kirchenbauer watermarks}, which we develop a subset of our identification mechanisms against.

\section{A Framework for Language Model Watermarks and Watermark Detection}
\label{sec:watermarking-tech}

Before introducing our identification algorithms, we first outline a framework and core terminology for the problem of identifying watermarks in LLMs.


\subsection{Large Language Models}

\begin{definition}[Vocabulary]
A \textit{vocabulary} in the context of LLM watermarking is a set of tokens $\mathcal{T}$ along with an encoder $\mathcal{E}$ and decoder $\mathcal{D}$ that encode and decode between sequences of tokens and text. For the vocabulary to be valid, we require $\mathcal{D}(\mathcal{E}(x)) = x$ for any string $x$.
\end{definition}

Notably, vocabulary encoders are sensitive to concatenation. That is, it is not always the case that $\mathcal{E}(x_1 x_2) = \mathcal{E}(x_1)\mathcal{E}(x_2)$ for strings $x_1, x_2$. As an example relevant to our identification algorithms, consider how numerals are tokenized in Google's \verb|Flan-T5| model. Suppose that $x = $ \verb|"5"| and $\mathcal{E}(x)$  is the token with index \verb|755|. However, $\mathcal{E}(xx)$ is not token \verb|755| repeated twice -- rather, it is a single token with index \verb|3769|. 

\begin{definition}[Large Sequence Model]
A \textit{large sequence model} $S$ over a vocabulary $\mathcal{T}$ is a map from a finite sequence of tokens $\mathcal{T}^*$ to a set of logits over all tokens $L \in \R^{|\mathcal{T}|}$, along with a sampler $R : \R^{|\mathcal{T}|} \to \Delta \mathcal{T}$ that randomly outputs a token based on the output logits.
\end{definition}


However, this definition does not capture a characteristic of LLM behavior that is critical for watermark identification. In all publicly hosted LLMs, the distribution over logits is highly uneven. Thus, we define a LLM as follows:




\begin{definition}[Large  Language Model]
A large sequence model is a \textit{large language model} if an adversary, only knowing the training data, is able to guess the sampled token with probability much greater than $1/|\mathcal{T}|$. 
\end{definition}


\subsection{Watermarks on Large Language Models}

With an understanding of relevant LLM mechanics, we now define a LLM watermark as follows:

\begin{definition}[Watermark]
A \textit{watermark} $W_s$ with secret key $s$ over a vocabulary $\mathcal{T}$ is a map $W_s : \mathbb{L}_\mathcal{T} \to \mathbb{L}_\mathcal{T}$, where $\mathbb{L}_\mathcal{T}$ is the set of LLMs with vocabulary $\mathcal{T}$. 
\end{definition}

\begin{definition}[Principled Watermark]
    
Let $\mathcal{L_A}, \mathcal{L_B}$ be large language models that are identical up to token permutation. If, for any pair of such LLMs and all keys $s$, $W_s(\mathcal{L_A})$ is identical to $W_s(\mathcal{L_B})$ up to the same token permutation, then each $W_s$ is a \textit{principled watermark}.
\end{definition}

All existing LLM watermarks that we are aware of satisfy this definition. However, watermarks that obey this definition are not necessarily useful. For instance, the identity watermark is a valid principled watermark, but is not algorithmically detectable within a sequence of text. We therefore introduce the following notion of watermark \textit{detectability} in text:


\begin{definition}[Detectability]
A watermark $W_s$ is $(p, P)$-\textit{detectable} for a model $\mathcal{L}$, some expression $P$, and $p \in (0, 0.5]$, if there exists a detector 
$D_s : \mathcal{T}^n \times \mathcal{T}^n \to \{0, 1\}$
that runs in $P(n)$ time and correctly distinguishes between sequences of length $n$ generated by $\mathcal{L}$ and $W_s(\mathcal{L})$ with probability at least $\frac{1}{2} + p$.
\end{definition}

Critically, \textit{detectability} of watermarks in text is different from the notion of \textit{identifying} models that have been watermarked, the central aim of this work.



A watermark should ideally not materially change the quality of LLM text generation. While quality is somewhat subjective, if it is impossible to distinguish watermarked text from standard text generated by a LLM, then the watermark must not affect any perceivable metric of quality. We use this observation to craft the following definition:

\begin{definition}[Strong Quality-Preserving Watermark]
Let $W_s$ be a watermark where $s$ has length $m$ and $\mathcal{L}$ is a LLM, $p$ and $q$ are polynomials, $n \le p(m)$, and $q(m) \geq 2$. Consider a $p(m)$-time adversary $A$ which takes in text and classifies it as watermarked or benign. $W_s$ is quality-preserving if, for all $\mathcal{L}, m, n, A, p, \text{ and } q$, $A$ is correct with probability at most $1/2 + 1/q(m)$ when given texts of length $n$ generated by $\mathcal{L}$ and $W_s(\mathcal{L})$, over the randomness of LLM generation and the choice of $s$. 


\end{definition}

This definition is more than sufficient for a watermark to preserve the quality of a text. None of the existing watermarks discussed here satisfy this standard of quality preservation, despite being relatively quality-preserving in practice.

For a deterministic LLM, strong quality-preservation and detectability conflict. The only way to be strong quality-preserving is to almost never modify the benign output, in which case the watermark is undetectable. The same is not true for non-deterministic language models.

\begin{theorem}
Assuming the existence of one-way functions, there exists a detectable watermark which is strong quality-preserving for non-deterministic LLMs.
\end{theorem}

\begin{proof}
Let $f_s$ be a pseudorandom function, which exists as one-way functions exist. Consider the watermark $W_s$ from \citet{kirchenbauer2023watermark} that generates a pseudorandom number $r \in [0,1]$ by applying $f_s$ to the previous tokens. The next token is then chosen by using $r$ to select the next token from the LLM logits.

This is strong quality-preserving, as otherwise an adversary that could distinguish a watermarked from unmarked language model could be used to distinguish $f_s$ from a random function. Since $W_s$ is deterministic for any given seed, it can be detected by rerunning the watermarked LLM and observing if it returns the same output.
\end{proof}

Such a watermark is detectable, and perfectly preserves quality, though it fails the desideratum that watermarks should still be detectable after the text is modified slightly. We will not formalize this desideratum in this paper.

Though other watermarks are less sensitive to changes to the text, all known watermarks are vulnerable to attacks that preserve generated text quality while evading detection \cite{sadasivan2023can}. As such, watermarkers might have an incentive to hide their watermarking algorithms or even the fact that they use a watermark.


\begin{definition}[Measurable Watermark]
Consider the following game played by an polynomial time (in $|s|$) distinguisher $A$ who has black-box access to a language model that is possibly watermarked. Suppose we have two generated texts from a model $\mathcal{L}$ and watermarked model $W_s(\mathcal{L})$. The adversary wins if it can determine which text is watermarked. 
The watermark is \textit{measurable} if there exists $A$ such that the probability of the adversary winning is at least $1/2 + 1/p(|s|)$ where $p$ is some polynomial.
\end{definition}



Detectability is at odds with immeasurability. The easier it is for a detector with access to underlying watermark seed to detect the watermark, the easier it is for a detector without access to the seed to detect it. This conflict is provable. In fact, for watermarks with a given detectability, there is a single adversary that can detect all of them.

\begin{theorem}\label{thm:detectability}


    Let $R$ be a sequence model which always samples uniformly from $\{0,1\}$. Denote $R^n$ as output text generated from $R$ with length $n$, and $W_s(R)^n$ similarly. 
    
    There exists an adversary $A$ such that, for any watermark $W_s$ and detector $D : \{0,1\}^n \to \{0,1\}$ such that $\Pr[D(R^n) = 0]/2 + \Pr[D(W_s(R)^n) = 1]/2 \ge \frac 12 + p$, $A$ can, with probability at least $1 - \Delta$, distinguish $R$ and $W_s(R)$ in $O(n\log(\frac{n}{p\Delta}\log(\frac{1}{\Delta})))$.
\end{theorem}

\begin{proof}
Consider the distribution of next-token probabilities for $0$ in a text.




If the detector correctly distinguishes between positive and negative distributions is at most $\frac{1}{2} + p$, we can use the bound from \citet{sadasivan2023can} to bound the total variation distance between $R^n$ and $W(R)^n$:
\[
\frac{1}{2} + p \le \frac{1}{2} + \TV(R^n, W(R)^n) - \frac{\TV(R^n, W(R)^n}{2} \]\[
\Rightarrow \TV(R^n, W(R)^n) \ge 1 - \sqrt{1 - p}
\]

Suppose that the adversary has the ability to not only sample the generator, but also obtain its probabilities for the next token. Consider the average variation distance from uniform of the next token from the watermarked generator, over a uniformly random in $(0, \dots, n-1)$  number of uniformly randomly generated previous tokens. By the subadditivity of the total variation measure, the average variation distance must be at least $\frac{1 - \sqrt{1-p}}{n}$. Since it is bounded in $[0, 0.5]$, at least $\frac{2 - 2\sqrt{1-p}}{n}$ of the sampled probabilities must be at least $\frac{1 - \sqrt{1-p}}{n}$. To ensure that with probability $1 - \Delta/2$ the adversary has sampled at least one such probability, it must take at least $m$ samples, with $(1 - \frac{2 - 2\sqrt{1-p}}{n})^m \le \Delta/2$, and so 
\[
m \le \frac{\log(\Delta/2)}{\log(1 - \frac{2 - 2\sqrt{1-p}}{n})}
\]

Since the adversary cannot sample probabilities, it must repeatedly sample a certain token. Let $k$ be the number of samples it takes from each particular token, and let the adversary classify the sample depending on whether the proportion of `0' generations differs from $1/2$ by at least $q$. Using the two-sided Chernoff bounds, we can get the probability of any particular sample from the uniform generator being misclassified. We then use union bounds to get the total probability of the uniform generator being misclassified, and use it to get bounds on $k$ and $q$:
\begin{align*}
    2m\exp(-k((0.5 + q)\log(1 + 2q) \\ + (0.5 - q)\log(1 - 2q))) \le \Delta
\end{align*}

We use a similar method to obtain the probability that a token with variation from uniform $v = \frac{1 - \sqrt{1-p}}{n}$ avoids detection:
\begin{align*}
\exp(-k((0.5 + q)\log(\frac{0.5 + q}{0.5 + v}) \\ + (0.5 - q)\log(\frac{0.5 - q}{0.5 - v}))) \le \Delta/2
\end{align*}
To get the $q$ which requires the fewest samples, we set these bounds to be equal. Doing this, we get a detection algorithm polynomial that takes $O(n\log(\frac{n}{pv\Delta}\log(\frac{1}{\Delta})))$ which correctly classifies both the random and any watermarked generator with probability at least $1 - \Delta$. 

\end{proof}

The fallout of this theorem is that, when the natural distribution of a language model acting on a fixed prompt is known, it cannot be watermarked undetectably. This forms a theoretical foundation for our identification mechanisms. 






\section{Understanding Language Model Output and Probability Distributions}

A watermark can be characterized and detected by how it affects the logits distribution of the underlying LLM. As such, our algorithms for watermark detection are centered heavily on analyzing shifts in language model output as well as logit and probability distributions. Therefore, it is critical to gain intuition for how these distributions usually behave.

\begin{figure}[h]
\vskip 0.2in
\begin{center}
\centerline{\includegraphics[scale=0.395]{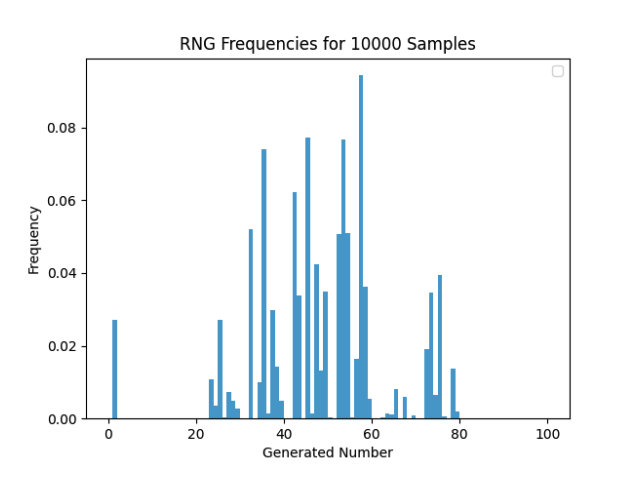}}
\caption{Example of a 10,000-sample RNG distribution generated by Alpaca-LoRA. Clearly, the distribution is far from uniform and exhibits idiosyncratic generations resulting from the training set.}
\label{fig:rng-example}
\end{center}
\vskip -0.2in
\end{figure}
\subsection{Random Bit Generation}

A simple case casts LLMs as random bit generators. Ideally, a LLM can generate bits uniformly at random when prompted, and so the identification mechanism in Theorem \ref{thm:detectability} would apply. We attempted random bit generation with OpenAI models. For each model, we use the following prompt:
\begin{verbatim}
"""Choose two digits, and generate a 
uniformly random string of those digits. 
Previous digits should have no influence 
on future digits:"""
\end{verbatim}
This prefix was followed by a fixed sequence of 20 `0's and `1's, produced by a Python random number generator. 

We let each model generate 100 tokens. For each token, if `0' and `1' were both a top-5 likely token, the probability of generating `0' was recorded. This procedure was repeated across multiple generations. A graph of the recorded probabilities for each model is displayed in Figure \ref{fig:bi-hist}. 


These distributions fail tests for normality and unsurprisingly, the corresponding generated bits are far from uniform. Surprisingly, the qualitative output probability distributions of each model are strikingly different. Ada (\ref{fig:sub1}) produces a roughly monotonically increasing distribution, babbage (\ref{fig:sub2}) produces a roughly truncated normal distribution, curie (\ref{fig:sub3}) produces a distribution with probability mass concentrated around $0$ and $1$, and davinci (\ref{fig:sub4}) produces a trimodal distribution with peaks around $0$, $0.5$, and $1$.

\begin{figure}[htbp]
\centering
\begin{subfigure}{.35\textwidth}
  \centering
  \includegraphics[width=\linewidth]{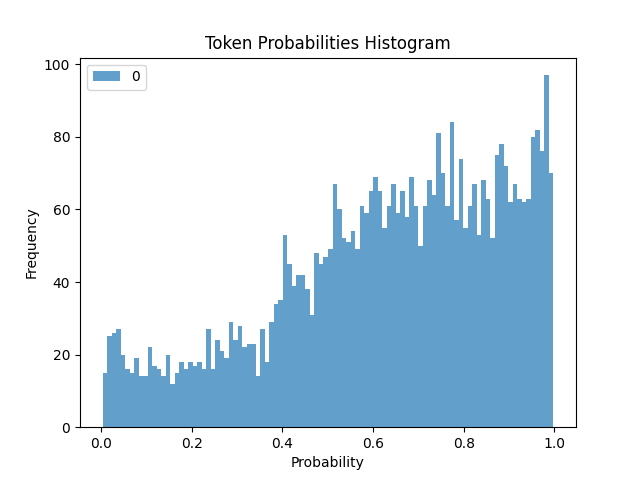}
  \caption{Ada probabilities of generating 0.}
  \label{fig:sub1}
\end{subfigure}
\begin{subfigure}{.35\textwidth}
  \centering
  \includegraphics[width=\linewidth]{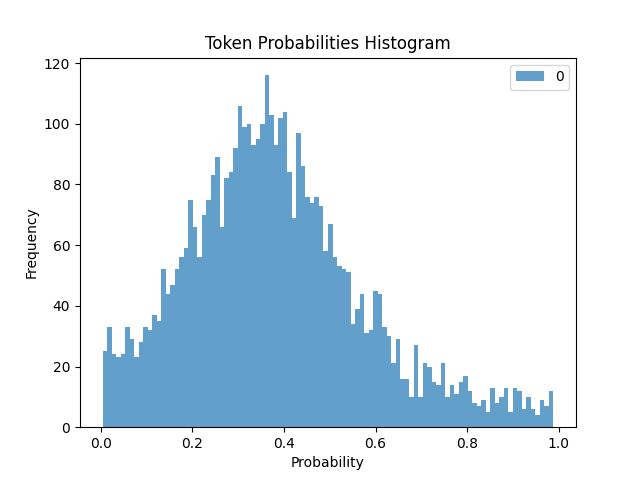}
  \caption{Babbage probabilities of generating 0.}
  \label{fig:sub2}
\end{subfigure}
\begin{subfigure}{.35\textwidth}
  \centering
  \includegraphics[width=\linewidth]{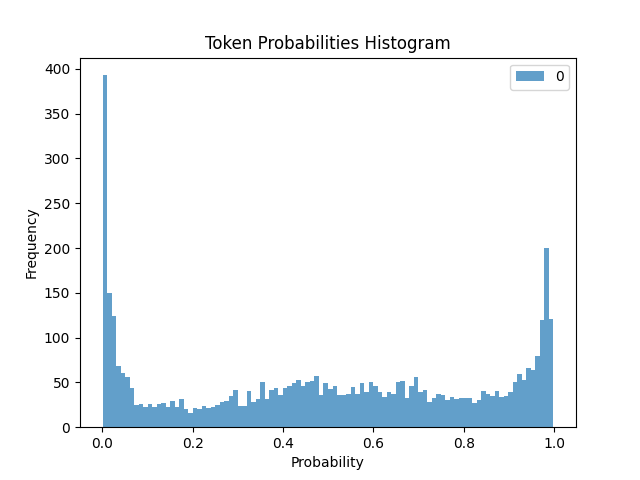}
  \caption{Curie probabilities of generating 0.}
  \label{fig:sub3}
\end{subfigure}
\begin{subfigure}{.35\textwidth}
  \centering
  \includegraphics[width=\linewidth]{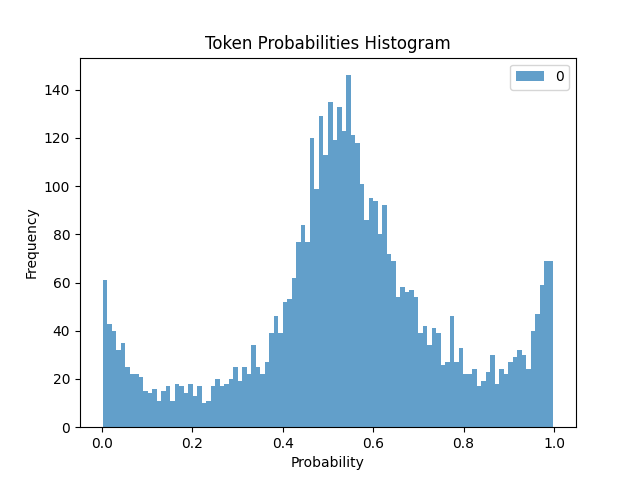}
  \caption{Davinci probabilities of generating 0.}
  \label{fig:sub4}
\end{subfigure}
\caption{Comparison of OpenAI engine behavior on a simple random bit generation task. The $x$-axis displays the retrieved probability of generating $0$, and the $y$-displays the frequency of each probability bucket over multiple generations.}
\label{fig:bi-hist}
\end{figure}


\begin{figure*}[t]
\vskip 0.2in
\begin{center}

\subfloat[\centering Lorenz curve for a unmarked Flan-T5-XXL language model. Most of the probability mass is concentrated in a few top tokens, as visualized by the sharp spike towards the right of the Lorenz curve.] {{\includegraphics[scale=0.375]{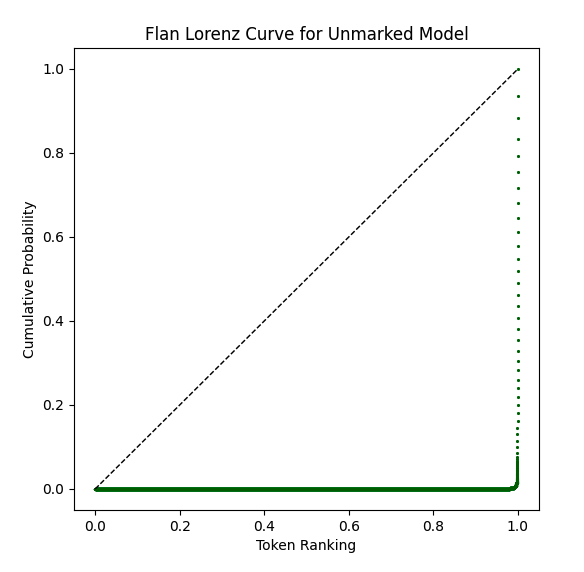}
}}
\qquad
\subfloat[\centering Lorenz curve for a Flan-T5-XXL model affected by a Kirchenbauer watermark with parameters $\gamma = 0.5$ and $\delta = 100$. Notice that the Lorenz curve is slightly smoother under this setting, due to the $\delta$ application on low-probability tokens.] {{\includegraphics[scale=0.375]{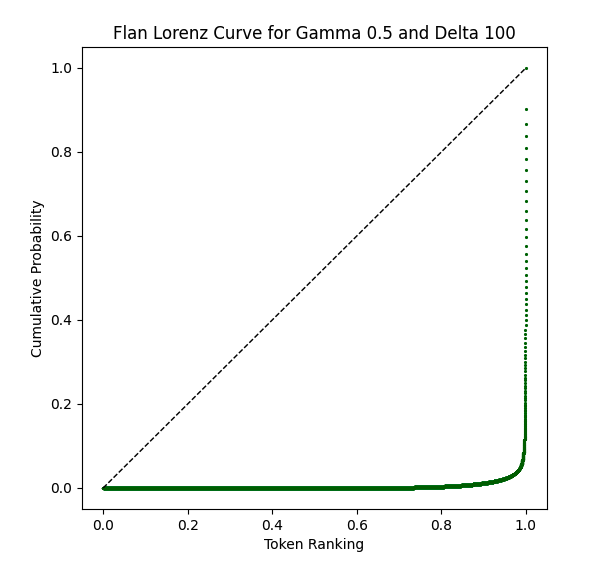}
}}

\caption{Examples of ranked probability Lorenz curves of the first token generated by Flan-T5-XXL under different Kirchenbauer watermarking strengths. The dashed line represents a perfectly uniform distribution. In both watermarking settings, the majority of the probability mass is concentrated in the top few tokens.}

\label{fig:lorenz}
\end{center}
\vskip -0.2in
\end{figure*}
\subsection{Ranked Probability Lorenz Curves}
\label{sec:ranked-prob-lorenz-curves}


Inspired by tools from econometrics, we use the Lorenz curve as a means of understanding language model behavior. Specifically, we examine the output token probabilities of a model and construct \textit{ranked probability Lorenz curves}. The $x$-axis of a ranked probability Lorenz curve lists the tokens sorted from lowest to highest probability, and the $y$-axis of the curve displays the probabilities of each token. Due to the sorted construction of the $x$-axis, the ranked token Lorenz curve is monotonically increasing. Figure \ref{fig:lorenz} displays an example of these Lorenz curves.


The Lorenz curve is an effective tool for understanding the effects of a watermark from \citet{kirchenbauer2023watermark}. Such a watermark adds a constant term $\delta$ to a randomly selected subset of green list token logits. In the ranked token Lorenz curve, this is notably reflected by a smoothing effect, as seen on the right of Figure \ref{fig:lorenz}. This indicates that a portion of low-probability tokens have experienced a $\delta$-increase.

To rigorize this notion of smoothness, one can compute the Gini coefficient $G$ of the Lorenz curve:
\begin{center}
    $\displaystyle  G = \frac{\sum_{i = 1}^n\sum_{j = 1}^n |x_i - x_j|}{2n^2 \overline{x}}$
\end{center} Here $x_i, x_j$ are the probabilities of $i$-th and $j$-th tokens on the curve, indexed by the ordered ranking, and $\overline{x}$ is the average probability. Traditionally in economics, $G$ is used to measure the inequality of a distribution. High $G$ suggests more inequality, reflected in unmarked distributions, while low $G$ suggests less inequality and a smoother distribution, suggesting the presence of a watermark.

\paragraph{Recovering Logits from Sampling} In practice, exact logits may not be available for analysis, for example when interacting with ChatGPT. In this case, we approximate token probabilities by sampling a large number of tokens from a language model, and calculating empirical probabilities.



\subsection{Random Number Generation}
\label{sec:random-number-generation}
In the case of a publicly hosted API, oftentimes logit data is not directly accessible. As a suitable approximation, we instead consider the distribution of tokens from a small subset of the original vocabulary. This enables us to analyze the shifting behavior of a LLM before and applying a watermark, without requiring access to output logits.

Specifically, we treat LLMs as random number generators, asking them to generate integers from 1 to 100, inclusive. Figure \ref{fig:rng-example} displays an example 10,000-sample distribution from Alpaca-LoRA using the following prompt:
\begin{verbatim}
"""Below is an instruction that 
describes a task. Write a response that 
appropriately completes the request.

### Instruction:
Generate a random number between 
1 and 100.

### Response:"""    
\end{verbatim}

While this is a natural task to restrict the output token set of a model, it is certainly not the only task that would do so. For example, asking a LLM to provide a synonym for a given input word, such as ``intelligent'', that starts with a specific letter, such as ``c'', would also severely restrict the output distribution to a subset resembling something like \{``clever'', ``canny'', ``crafty'', ``calculating'', ``cunning''\}. 

A key benefit of the random number generation task over other alternatives, however, is that the output space for any model is fairly consistent between models, generating integers between 1 and 100, regardless of model capacity. While the distribution of numbers is certainly expected to change across models, the range of outputs is relatively more stable.



\begin{figure*}[h]
    \centering
    \subfloat[\centering Logit distribution produced by an unmarked Alpaca-LoRA model. The tokens are indexed by the original tokenizer ordering. While there is certainly a wide variation in logit values, there is no distinct separation.] {{\includegraphics[scale=0.53]{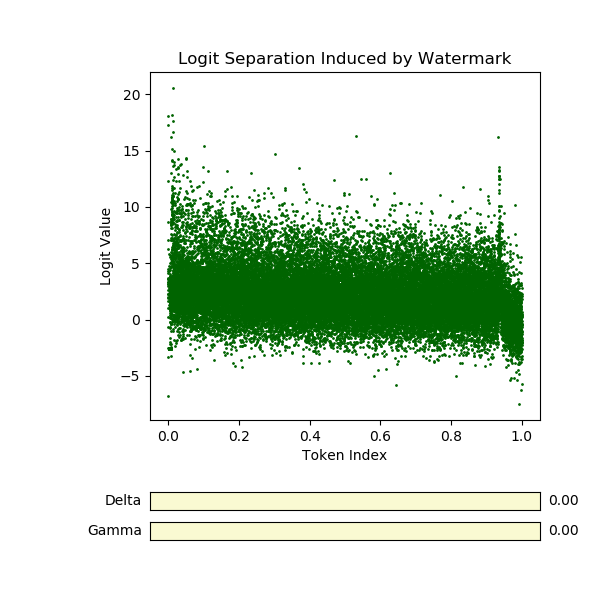}
    }}
    \qquad
    \subfloat[\centering Logit distribution produced by a Alpaca-LoRA model with a $\gamma = 0.15$ and $\delta = 40$ Kirchenbauer watermark. There is a distinct logit separation of size $\delta$ into two bands, one for the original logits, and one for logits perturbed by $\delta$. The width of the top band corresponds to $\gamma$.] {{\includegraphics[scale=0.53]{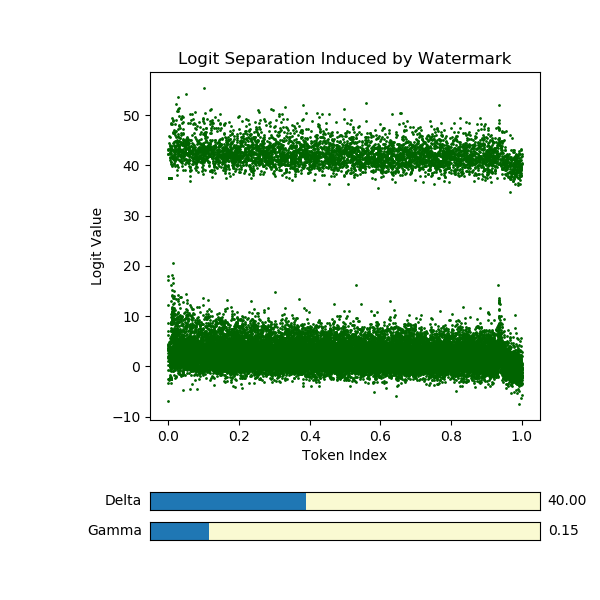}}}
    \caption{Logit separation induced by Kirchenbauer watermarks.}
    \label{fig:logit-separation}
\end{figure*}

\section{Baseline Mechanisms for Identifying Watermarked Large Language Models}

Here, we introduce three simple watermark detection algorithms based on analyzing exact and approximate probability and logit distributions. Critically, our algorithms do not require any access to information governing the underlying watermark generation procedure, such as a hash function or random number generator. We hope these algorithms can serve as sound baselines for future work in this field.

Our three proposed algorithms vary in their access to exact versus sampled logits, generalizability across watermarking schemes, and statistical robustness. Depending on the objective and identification constraints, such as efficient computation, interpretable test statistic, availability of logits, and robustness to random shifts in the data, a different algorithm will be ideal.

\subsection{Measuring Divergence of RNG Distributions}
\label{sec:kolmogorov}

The first algorithm is centered on the simple idea of measuring divergence in ``random'' number distributions generated by a LLM, as alluded to in \S \ref{sec:random-number-generation}. In particular, we make use of the Two-sample Kolmogorov–Smirnov test to determine whether an empirical random number distribution of a \textit{watermarked} LLM shifts from an empirical random number distribution of an \textit{unmarked} LLM.



Given a specific LLM, we first generate a 1000 number empirical distribution $F_{u, n}$ as described in \S \ref{sec:random-number-generation} using an \textit{unmarked} model. We watermark the LLM and produce an empirical distribution $F_{w, m}$ in the same fashion. We then compute the Kolmogorov-Smirnov statistic as follows:
\begin{center}
    $D_{n , m} = \sup_{x} |F_{u, n} (x) - F_{w, m}(x)|$
\end{center} Here $n$ and $m$ are the sizes of each sample, and $n = m = 1000$ specifically in our case. 

The null hypothesis is that the samples are drawn from the same distribution, i.e.:
\begin{center}
    \textit{$H_0$: $F_{u, n}$ and $F_{w, m}$ are drawn from the same underlying distribution }
\end{center}

We reject $H_0$ at significance level $\alpha$ if:
\begin{center}
    $D_{n, m} > c(\alpha) \sqrt{\frac{n + m}{n \cdot m}} $
\end{center} Here $c(\alpha) = \sqrt{-\ln{\left ( \frac{\alpha}{2} \right)} \cdot \frac{1}{2}}$.








\subsection{Mean Adjacent Token Differences}
\label{sec:mean-adjacent-token}

From the Lorenz curve and Gini measure discussed in \S \ref{sec:ranked-prob-lorenz-curves}, a natural extension is to analyze the average increase in logit value between adjacent tokens. That is, we compute:
\begin{center}
    $\displaystyle \mathcal{I} = \frac{\sum_{i = 1}^{n - 1} \ell_{i + 1} - \ell{i}}{n - 1} $
\end{center} Here, $\ell_i$ is the logit at index $i$ on the Lorenz curve, and $n$ is the total number of tokens in the vocabulary.

Note that for a Kirchenbauer-watermarked LLM with logit perturbation $\delta$ and green list $G$ with proportion $\gamma$, we have an average logit increase of: 
\begin{align*}
    \mathcal{I_{W}} &= \frac{\sum_{i = 1}^{n - 1} (\ell_{i + 1} - \ell{i}) \mathds{1}[i \in G] }{n - 1} \\
    &= \frac{\gamma (n - 1) \delta + \sum_{i = 1}^{n - 1} (\ell_{i + 1} - \ell{i})}{n - 1}
\end{align*}

Taking the difference with the average logit increase of an unmarked model, $\mathcal{I_U}$, we have:
\begin{align*}
    \mathcal{I_W} - \mathcal{I_U} = \frac{\gamma (n - 1) \delta}{n - 1} = \gamma \delta
\end{align*}

Taking the above $I_W - I_U$ as inspiration, a simple identification procedure is to periodically compute $\mathcal{I}$ and observe how it varies over time. Notice that $I_W - I_U$ directly varies with $\gamma$ and $\delta$; that is, the strength of the watermark directly influences its detectability. For a strong watermark, variations in $\mathcal{I}$ will be obvious, while weaker watermarks will manifest subtler differences in $\mathcal{I}$.




\begin{table*}[t]
\caption{Tradeoffs between proposed watermarking identification algorithms. An ideal watermarking is not specific to Kirchenbauer and can detect \textbf{general watermarks}; does \textbf{not require access to logits}, which is common in publicly hosted models; is \textbf{sensitive to small $\delta$} and parameter values; is \textbf{robust against other distribution shifts} not induced by watermarks; and can be \textbf{performed in a single snapshot of time} without reference to previous distributions or tests.}
\label{tab:tradeoff}
\vskip 0.15in
\begin{center}
\begin{small}
\begin{sc}
\begin{tabular}{lccccc}
\toprule
Detection Method & General Watermarks & Logit-Free & $\delta$-Sensitive & Shift-Robust & Single-Shot \\
\midrule
RNG Divergence & $\surd$ & $\surd$ & $\surd$ & $\times$ & $\times$ \\
Mean Adjacent & $\times$ & $\times$ & $\surd$ & $\times$ & $\times$ \\
$\delta$-Amplification & $\times$ & $\times$ & $\surd$ & $\surd$ & $\surd$ \\
\bottomrule
\end{tabular}
\end{sc}
\end{small}
\end{center}
\vskip -0.1in
\end{table*}

\begin{table*}[bp]
\caption{$\delta$-Amplification algorithm produces the corresponding bimodality test dip and $p$-values for a small-$\delta$ watermarked Alpaca-LoRA model when using prompt prefixes randomly sampled from the Pile and OpenWebText datasets. Greater diversity in prompt prefix task and content increasingly induces bimodality and thus watermark identification. Notice that when only using Pile prompts, $\delta$-Amplification is only identify a watermarked model at strength $\delta = 7$, compared to $\delta = 5$ when using both OWT and Pile prompts.}
\centering
\vskip 0.15in
\begin{small}
\begin{sc}
\begin{tabular}{ccccccc}
\toprule
$\delta$ & P-Value (OWT \& Pile) & Dip (OWT \& Pile) & $p$ (Pile) & Dip (Pile) \\
\midrule
0 & 0.886 & 0.0017 & 0.908 & 0.0016 \\
1 & 1.0 & 0.00094 & 0.999 & 0.00097 \\
2 & 0.991 & 0.0013 & 1.0 & 0.00092 \\
3 & 0.900 & 0.0016 & 1.0 & 0.00093 \\
4 & 0.204 & 0.0025 & 1.0 & 0.00093 \\
5 & 0.0 & 0.00497 & 1.0 & 0.00093 \\
6 & 0.0 & 0.0087 & 0.947 & 0.0015 \\ 
7 & 0.0 & 0.012 & 0.0 & 0.0048 \\
8 & 0.0 & 0.017 & 0.0 & 0.013 \\
9 & 0.0 & 0.023 & 0.0 & 0.025 \\
10 & 0.0 & 0.033 & 0.0 & 0.037 \\ 
\bottomrule
\end{tabular}
\end{sc}
\end{small}
\label{tab:dip-p-values-combined}
\end{table*}

\subsection{Robustly Identifying Small-$\delta$ Watermarks}
\label{sec:delta-amp}


While \S \ref{sec:mean-adjacent-token} introduces a metric that will successfully detect a Kirchenbauer watermark for small $\delta$, it is sensitive to general logit distribution perturbations introduced by other scenarios, such as routine model updates. An identification method robust to general distribution shifts should rely on shift characteristics specific to a Kirchenbauer watermark.

Notably, a Kirchenbauer watermark will induce perceptible band separations in logit space. Figure \ref{fig:logit-separation} demonstrates an example of this phenomenon. Inspired by this observation, we draw an analogue between the separation of logit values into bands and the bimodality of logit frequencies. Under this reframing, testing for bimodality is equivalent to testing for the existence of a band gap. 

However, though this approach is robust to other distribution shifts, it does not yet consider small-$\delta$ perturbations. To handle such situations, we introduce the $\delta$-Amplification algorithm.


\begin{algo}[$\delta$-Amplification]
    Suppose we have a potentially watermarked LLM $\mathcal{L}$. We wish to detect if it is watermarked. We prompt $\mathcal{L}$ repeatedly as follows:
\begin{verbatim}
[Random string sampled from training 
datasets]. Now write me a story: 
\end{verbatim} Take the produced logits and average them across repetitions. If the resulting frequency of averaged logits is bimodal, conclude that $W_s(\mathcal{L})$ is watermarked. 

To recover the underlying Kirchenbauer watermark parameters, we estimate $\delta$ by measuring the distance between the peaks, and $\gamma$ by measuring their respective masses.
\end{algo}

Critically, as watermarks \citep{kirchenbauer2023watermark, aaronson2023} only use a fixed-size previous token window (rumored to be 5-tokens in OpenAI models) to determine green list indices, the green list partition across all prompts is the same under this algorithm, as every prompt ends in a fixed ``Now write me a story:'' suffix. Therefore, the output logit distributions all experience the same $\delta$ mask.

However, the model is still influenced by earlier tokens in the prompt, and thus exhibits differing logit values across prompts. Intuitively then, averaging distributions across different prompts reduces the variation of logits, but maintains the same effect of the $\delta$ perturbation. The averaged distribution thus amplifies the effects of a small-$\delta$ watermark. Figure \ref{fig:delta-amp-example} demonstrates an example of this effect.

We test for bimodality via the Hartigan dip test \cite{hartigan1985dip}. For a distribution with probability distribution function $f$, this test computes the largest absolute difference between $f$ and the unimodal distribution which best approximates it. 
\begin{align*}
    D(f) = \newinf_{g \in U} \sup_{x} |f(x) - g(x)|
\end{align*}

\begin{figure*}[h]
    \centering
    \subfloat[\centering Distribution of logit values prior to $\delta$-Amplification.] {{\includegraphics[scale=0.53]{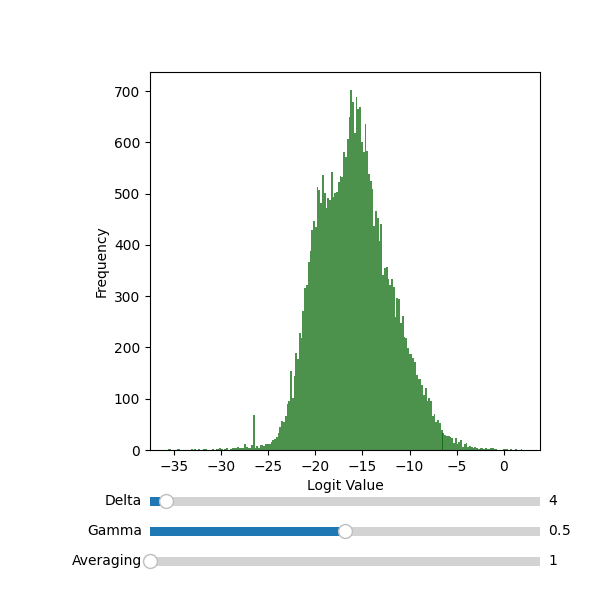}
    }}
    \qquad
    \subfloat[\centering Distribution of logit values after applying $\delta$-Amplification, averaging across 140 prompts.] {{\includegraphics[scale=0.53]{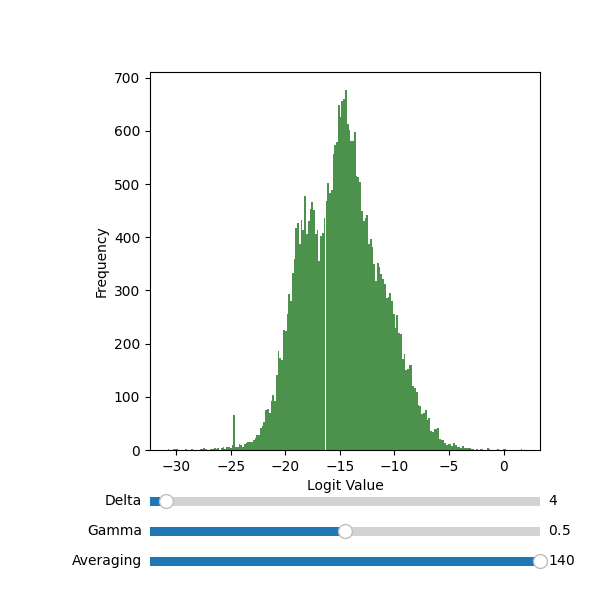}}}
    \caption{Distribution of Alpaca-LoRA logit values before and after $\delta$-Amplification for the ``Now write me a story:'' prompt. The $x$-axis is the logit value and $y$-axis is the frequency. Without $\delta$-Amplification, it is not clear whether the distribution of logit values exhibits bimodality. Bimodality emerges only after applying $\delta$-Amplification, enabling watermark identification.}
    \label{fig:delta-amp-example}
\end{figure*}

Here $U$ is the set of all unimodal distributions over $x$. The corresponding $p$-value is calculated as the probability of achieving a Dip score at least as high as $D$ from the nearest unimodal distribution.



\subsection{Tradeoffs Between Detection Algorithms}
The algorithms proposed above are all effective in different senses. \S \ref{sec:kolmogorov} introduced a RNG divergence approach to watermark detection that is not specific to Kirchenbauer watermarks, does not require access to logits and can thus be used directly on black-box public APIs, and is also sensitive to small $\delta$ watermarks. 

\S \ref{sec:mean-adjacent-token} introduced an adjacent token metric for analyzing Kirchenbauer watermarks that is sensitive to small-$\delta$.

Finally, we extended this approach in \S \ref{sec:delta-amp} to robustly handle small-$\delta$ watermarks, while preserving the single-shot criteria. Moreover, the $\delta$-Amplification approach lent itself nicely to statistical testing, specifically of bimodality. It is also the only identification method that can be performed in a single shot. That is, it does not require comparing behavior between a watermarked and unmarked model, and equivalently between a single model across multiple snapshots in time, as is the case for the previous two algorithms.

Each method has merit depending on the specific identification setting, but will also sacrifice certain desiderata. Table \ref{tab:tradeoff} summarizes these tradeoffs.

\subsection{Monitoring}
To detect watermarks in publicly hosted models, we set up monitoring scripts with our identification mechanisms that periodically query these models and compute the relevant tests and metrics. We are most interested in identifying watermarks in OpenAI models that are \textit{not} API-accessible, as we believe that UI-based versions of these models will be most susceptible to dishonest usage (e.g. students cheating will primarily use ChatGPT and not a Python script accessing API). As such, we also set up an agent to interact and monitor the UI-based version of these models.

\begin{table}[h!]
\caption{Kolmogorov-Smirnov test results on 1000-sample ``RNG`` distributions from models watermarked at varying strengths. A test is performed between 30 distributions generated from the model in each row against a random distribution generated from an unmarked model. The reported p-values are averaged across these 30 samples. Note that the first row of each section, where $\gamma = \delta = 0$, is the Kolmogorov-Smirnov test result between an unmarked model against an unmarked model.}
\label{tab:komolgorov}
\vskip 0.15in
\begin{center}
\begin{small}
\begin{sc}
\begin{tabular}{lcccr}
\toprule
Model & $\gamma$ & $\delta$ & Average P-Value \\
\midrule
Flan-T5-XXL & 0 & 0 & 0.80 \\
Flan-T5-XXL & 0.1 & 1 & 3.54e-7 \\
Flan-T5-XXL & 0.1 & 10 & 1.22e-9 \\
Flan-T5-XXL & 0.1 & 50 & 6.75e-7 \\
Flan-T5-XXL & 0.1 & 100 & 8.11e-8 \\
Flan-T5-XXL & 0.25 & 1 & 0.002 \\
Flan-T5-XXL & 0.25 & 10 & 6.47e-9 \\
Flan-T5-XXL & 0.25 & 50 & 2.59e-7 \\
Flan-T5-XXL & 0.25 & 100 & 1.33e-6 \\
Flan-T5-XXL & 0.5 & 1 & 0.00024 \\
Flan-T5-XXL & 0.5 & 10 & 0.057 \\
Flan-T5-XXL & 0.5 & 50 & 0.054 \\
Flan-T5-XXL & 0.5 & 100 & 0.054 \\
Flan-T5-XXL & 0.75 & 1 & 0.42 \\
Flan-T5-XXL & 0.75 & 10 & 0.22 \\
Flan-T5-XXL & 0.75 & 50 & 0.34 \\
Flan-T5-XXL & 0.75 & 100 & 0.23 \\
\addlinespace[0.75ex]
\hdashline
\addlinespace[0.75ex]
Alpaca-Lora & 0 & 0 & 0.63 \\
Alpaca-Lora & 0.1 & 1 & 1.40e-10 \\
Alpaca-Lora & 0.1 & 10 & 4.31e-37 \\
Alpaca-Lora & 0.1 & 50 & 4.31e-36 \\
Alpaca-Lora & 0.1 & 100 & 1.48e-35 \\
Alpaca-Lora & 0.25 & 1 & 3.10e-13 \\
Alpaca-Lora & 0.25 & 10 & 1.93e-10 \\
Alpaca-Lora & 0.25 & 50 & 4.52e-14 \\
Alpaca-Lora & 0.25 & 100 & 2.14e-11 \\
Alpaca-Lora & 0.5 & 1 & 4.17e-13 \\
Alpaca-Lora & 0.5 & 10 & 0.0089 \\
Alpaca-Lora & 0.5 & 50 & 0.066 \\
Alpaca-Lora & 0.5 & 100 & 0.06 \\
Alpaca-Lora & 0.75 & 1 & 0.00015 \\
Alpaca-Lora & 0.75 & 10 & 0.52 \\
Alpaca-Lora & 0.75 & 50 & 0.40 \\
Alpaca-Lora & 0.75 & 100 & 0.48 \\
\bottomrule
\end{tabular}
\end{sc}
\end{small}
\end{center}
\vskip -0.1in
\end{table}


\section{Results}
We perform experiments on our identification mechanisms using the Flan-T5-XXL and Alpaca-LoRA models due to their strong instruction-following capabilities but differing Byte-Pair Encoding and digit tokenization methods.

Table \ref{tab:komolgorov} displays the $p$-values resulting from the Kolmogorov-Test method. For each model and watermark strength, the method is performed across 30 independent instances of 1000-sample distributions generated from a Kirchenbauer-watermarked model. Specifically, we perform a test between each of the 30 distributions and a distribution generated by an unmarked model. Under this procedure, any model with distributions producing an average $p$-value less than 0.05 would be considered watermarked. The $p$-values are highest when comparing an unmarked distribution against an unmarked distribution, as expected. Notably, the $p$-values are extremely low for a majority of watermark strengths for both Flan-T5-XXL and Alpaca-LoRA. 

The results of the $\delta$-Amplification method and corresponding bimodality test are in Table \ref{tab:dip-p-values-combined}. Concretely, we sample a diverse range of prompt prefixes from Pile and OpenWebText via HuggingFace datasets and run tests on the logit value distributions from these generations. Notably, diversity in prompt prefix task and content enables uncorrelated variance in $\delta$, thus most effectively eliminating logit variance post-averaging. 


In particular, we observe that at $\delta \geq 5$, our method produces $p$-values less than 0.05, thus successfully identifying the presence of a watermark in the model. Critically, increasing the number of varied prefix prompts also increases identification potency. Namely, averaging logit distributions only across Pile prompts identifies Kirchenbauer-watermarked models only at strength $\delta \geq 7$, while averaging across both Pile and OpenWebText prompts identifies watermarked models at strength $\delta \geq 5$.

As such, both algorithms serve as strong baselines for watermark identification.




\section{Conclusion}
In this work, we develop a theoretical framework for understanding the watermark identification problem in large language models. We then provide three black-box baseline algorithms -- measuring divergence of RNG distributions, mean adjacent token differences in logits, and $\delta$-Amplification -- for identifying watermarks, which all fundamentally rely on the analysis of the distributions of model outputs, logits, and probabilities. Each algorithm trades off in different practical aspects, including identification generalizability, logit-free analysis, sensitivity to watermarks, robustness against general distribution shifts, and single-shot testing. Ultimately, we monitor publicly hosted models in an attempt to detect watermarks. Since we are the first to consider the problem of identifying watermarks in large language models, we hope that our framework and baselines serve as strong foundations for future work in this direction from the community.


\nocite{langley00}

\bibliography{paper}
\bibliographystyle{icml2023}

\newpage
\appendix
\onecolumn




\end{document}